\documentclass[conference]{IEEEtran}
\IEEEoverridecommandlockouts

\usepackage[T1]{fontenc}
\usepackage{cite}
\usepackage{amsmath,amssymb,amsfonts}
\usepackage{amsthm}
\usepackage{algorithmic}
\usepackage{graphicx}
\usepackage{textcomp}
\usepackage{xcolor}
\usepackage[hidelinks]{hyperref}
\usepackage{booktabs}
\usepackage{microtype}
\usepackage{enumitem}

\graphicspath{{figures/}{../figures/out/}{../figures/}}

\newtheorem{definition}{Definition}
\newtheorem{proposition}{Proposition}
\newtheorem{corollary}{Corollary}
\newtheorem{remark}{Remark}

\newcommand{\R}{\mathbb{R}}
\newcommand{\D}{\mathcal{D}}
\newcommand{\relu}{\rho}
\newcommand{\eps}{\varepsilon}
\newcommand{\epsstar}{\eps^{\!*}}
\newcommand{\ablate}{\mathrm{ablate}}
\newcommand{\vol}{\mathrm{vol}}

\begin{document}

\title{Certified Mechanistic Edits: Behavioral Guarantees for Skill Removal and Preservation}

\author{\IEEEauthorblockN{Md Sazid Uddin, Md.\ Khairul Alam Mazumder, M.\ F.\ Mridha}
	\IEEEauthorblockA{Department of Computer Science\\
		American International University-Bangladesh (AIUB), Dhaka, Bangladesh\\
		\{sazid.uddin, khairul, firoz.mridha\}@aiub.edu}}

\maketitle

\begin{abstract}
	Mechanistic edits (ablations, weight edits, activation steering) are the standard tools for unlearning a harmful capability from a neural network while preserving useful ones. Current approaches validate their effects only by testing, which can never cover an entire continuous region of inputs. Prior work at the interpretability–verification boundary certifies descriptions of a model: what a circuit computes, or whether it faithfully explains the whole. We instead certify the behavioral effect of an edit: that disabling a circuit removes one skill and provably preserves another, for every input in a region; a feature non-interference guarantee in the information-flow-security sense. We demonstrate such certified edits from toy ReLU networks up to a standard softmax + LayerNorm transformer, proving removal and preservation over continuous embedding-space regions and reaching roughly 9× the input-perturbation dimension an exact solver can handle by switching to sound bound propagation. Furthermore, we prove that no finite deterministic black-box test can certify removal, exhibiting an edit that passes exhaustive testing yet provably fails on a survivor pocket that can be made arbitrarily small. Guarantees hold on small, standard-architecture networks and, like any removal claim, presuppose that the target skill admits a decidable specification, a property which real-world harms may not have.
\end{abstract}


\section{Introduction}
\label{sec:intro}

Ablating a circuit, zeroing a weight, clamping a sparse-autoencoder feature, or adding a representation-steering vector now form the standard mechanistic interpretability toolkit for making models safer. These help with unlearning a dangerous capability, suppressing a behavior, or enforcing a refusal. In order to test whether an edit has the desired effect, a mechanistic practitioner runs the edited model on a held-out test set of prompts, checks that the target skill or behavior is gone, and that unrelated skills survive, and ships the edit if the test results are satisfactory.

This method of testing answers the question only at the inputs tested. This method of testing answers the question only at the inputs tested. Jailbreaks succeed precisely on inputs outside the distribution that safety training and evaluation covered~\cite{wei-jailbroken,zou-gcg}, and unlearning edits that pass their benchmark evaluations are reversed by such inputs~\cite{lucki-unlearning}. A guarantee that holds for every input in a continuous region would be strictly stronger as a guarantee, and it is exactly what a test set cannot provide. This paper supplies that guarantee. We certify, by formal proof, the \emph{behavioral effect of an edit}: that after the edit a specified skill is absent, and another is preserved, for \emph{all} inputs in a region of the model's input representation.

A growing body of work brings formal verification to interpretability, but it certifies \emph{descriptions} of a model: an accuracy floor for the unedited network~\cite{compact-proofs}, whether a discovered circuit faithfully and minimally explains the whole~\cite{hadad}, what a circuit computes~\cite{somani}, or whether a circuit is stable under dataset resampling~\cite{certified-circuits}. A separate line of work certifies unlearning statistically, indistinguishability from a model retrained without the deleted data~\cite{certified-unlearning}. However, these certificates do not specify how a given edit affects the model's behavior over a region of inputs. This area of experimentation has been unexplored, and this paper aims to fill this gap. In security terms, our guarantee is a \emph{feature non-interference} property: the edited component's output is provably independent of the forbidden input feature, everywhere in the region.

The stronger claim is not merely undertested by finite test suites; it is unreachable by them. We prove (Proposition~\ref{prop:p3}) that no finite deterministic black-box test (adaptive queries included) can certify removal of a target skill. The witness is concrete: for an edit that passes an exhaustive grid of tests, our solver exhibits a surviving input where the skill is intact (Fig.~\ref{fig:illusion}). Ordinary training produces such ``intervention illusions'' unprompted, in models as small as four inputs, with survivor pockets below $1.5\times10^{-6}$ of the domain. A diligent tester can therefore ship an edit that leaves the skill alive in a vanishing-but-nonzero pocket, unknowingly.
\subsection*{Contributions of this paper:}
\begin{enumerate}
	\item We formalize a certified edit and its certified radius, prove they compose and are cheaply computable (P1--P2), and prove that no finite deterministic black-box test can certify removal (P3), with a constructive witness that our code executes. A companion result (P4) separates edit classes: surgery vs.\ steering, exactly, with a certified collateral law.
	\item An exact-rational encoding into an SMT solver certifies removal and preservation over all token sequences $\times$ a continuous embedding ball, on piecewise-linear models up to a threshold-gate transformer, together with a map of exactly how far exact proofs reach (Table~\ref{tab:frontier}).
	\item A bound-propagation extension certifies the same kind of edit on a standard softmax\,+\,LayerNorm transformer that the exact solver cannot encode due to its non-linearity. The extension reaches roughly $9\times$ the perturbation dimension of the exact frontier, validated to \emph{equal} the exact radius where both apply and bracketed above by attack.
\end{enumerate}

\subsection*{Scope of the guarantees:}
Our guarantees hold on small, standard-architecture networks; the skills we remove/preserve have a decidable ground-truth specification, which is what makes ``removed'' a deterministic claim, a property that real harmful capabilities often lack. The demonstration is a proof of concept, motivated by the need for rigorous safety guarantees in AI systems. \S\ref{sec:limitations} states the scope in full, and \S\ref{sec:conclusion} names the follow-up of this work.

\begin{figure*}[t]
	\centering
	\includegraphics[width=\textwidth]{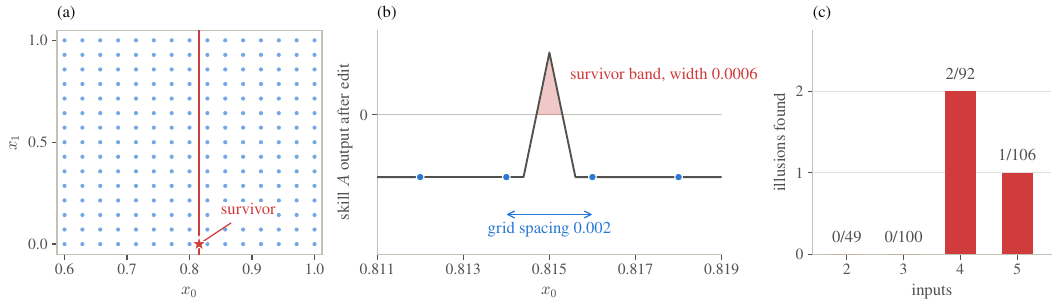}
	\caption{An edit that every test approves and the solver refutes. (a)~Skill~$A$'s removal region for a constructed two-input model after an ablation edit. A $15\times15$ grid (shown), a $201\times201$ grid of $40{,}401$ points and $300$ random inputs all report the skill removed while the solver nonetheless returns a surviving input (star). (b)~The edited skill-$A$ output along $x_0$ through the survivor: a tent of three ReLUs rises above zero in a band $0.0006$ wide, between neighboring grid points $0.002$ apart (Proposition~\ref{prop:p3}). (c)~An automated search over trained models, $80$ per input dimension, finds no illusion with two or three inputs and three with four or five; their survivor regions lie below $1.5\times10^{-6}$ of the domain (95\% Clopper--Pearson bound).}
	\label{fig:illusion}
\end{figure*}

\section{Related Work: the description axis vs.\ the edit axis}
\label{sec:related}

We organize the relevant literature by the object each method certifies. Prior formal work certifies objects such as the description of a fixed model; a statistical line certifies an intervention's average effect; we certify the worst-case behavioral effect of an edit over a continuous region. Table~\ref{tab:positioning} summarizes the axis.

\paragraph{Certified interpretability (the description axis)}
Compact proofs \cite{compact-proofs} certify an accuracy bound of the unedited model. \emph{Formal mechanistic interpretability}~\cite{hadad} certifies that a discovered circuit is faithful and minimal; its Definition~2 quantifies over continuous activation-patch values at finitely many reference inputs, but constrains the agreement with the unedited model, not the task behavior of an edit. Certified Circuits~\cite{certified-circuits} certifies the stability of a discovered circuit under dataset resampling. The nearest neighbor, Verifiable Transformers~\cite{somani}, uses an SMT solver to certify four descriptive properties of a circuit explanation (what it computes, per-input edge necessity, robustness of the final residual); it explicitly disclaims statements about an edit's behavioral effect. Its ``model editing'' is surgery performed to make the model verifiable by stripping LayerNorm, replacing heads with restricted programs, then freezing and hashing. It is however not a claim about what an edit does. These methods answer ``is this explanation of the model faithful?''; we answer ``what happens when you act on it?''

\paragraph{Certified edits of a different kind}
\emph{BlockCert}~\cite{blockcert} attaches certificates to weight edits, but the certified statement is a bound on $\lVert F'(x)-F(x)\rVert$ over a finite traced prompt set demonstrating how far outputs move, not what the behavior became. A large certified deviation is compatible with the target skill surviving, and a small one with collateral damage. Finite prompt evidence is, moreover, exactly the protocol class our Proposition~\ref{prop:p3} proves cannot certify removal. We view BlockCert as complementary infrastructure for a different question to our own. Certified unlearning~\cite{certified-unlearning} gives a statistical, training-data-level guarantee that the unlearned model is distributionally indistinguishable from one retrained without the deleted data. It certifies a \emph{process} relative to a counterfactual model, not the \emph{behavior} of the result; the two guarantees are orthogonal and composable.

\paragraph{The statistical / anytime-valid axis}
A parallel line certifies interventional claims \emph{statistically}, over a distribution rather than a worst-case region. Certified interventional fidelity (CIF)~\cite{cif} writes the reported quantity as a causal estimate and gives anytime-valid confidence sequences robust to adaptive intervention sampling. Statistical unlearning of distributions~\cite{stat-unlearn} characterizes a removal-preservation Pareto frontier between data distributions. Both are average-case; neither certifies that a specified behavior is absent for every input in a continuous region of one fixed model. CIF is, in fact, precisely the randomized-adaptive-tester case that Proposition~\ref{prop:p3} explicitly declines to cover. The two are complementary halves of one worry, and we distinguish worst-case/region/exact from average-case/distribution/anytime-valid explicitly.

\paragraph{Black-box impossibility neighbors}
Two methodologically relevant recent results deal with proving impossibility theorems of black-box models. Hair-Trigger Alignment~\cite{hair-trigger} proves that static black-box probing cannot distinguish a genuinely robust model from one hiding adversarial behavior that a single benign fine-tuning update re-activates. This is a statement about a future weight update, not about whether a test certifies that a fixed edit removed a skill over a region. Unlearning as distribution restoration~\cite{yang-yeung} gives an empirical finite-query impossibility for oracle-free unlearning certification; P3 sharpens the deterministic-adaptive case it leaves open. An independent, non-interpretability instance of the same epistemic point, verification over a continuous region catching failures that testing \emph{between} cases misses appears in autonomous-vehicle control~\cite{testing-between}.

\paragraph{Non-interference as the security frame}
We state removal as \emph{feature non-interference}: the output is provably independent of a forbidden input feature. The vocabulary is borrowed from LLM security. GIF~\cite{gif} certifies geometric information-flow control for LLM agents (zero flow of untrusted spans to sensitive actions) but is applied there to data flow through an agent. To our knowledge we are the first to certify non-interference as the effect of a mechanistic edit.

\paragraph{Editing practice and its pitfalls}
The steering and unlearning literature ships interventions with empirical validation only~\cite{repe}, and safety analyses document their collateral damage: even random-direction steering breaks safety~\cite{rogue-scalpel}, and refusal ablation has statistically significant, sometimes sign-reversing off-target effects at frontier scale~\cite{abliteration,steering-pitfalls}. Our Proposition~\ref{prop:p4} gives, to our knowledge, the first exact account of that trade-off for a model class, and \S\ref{sec:findings} turns it into certified radii. That steering can be compiled into weight changes is established practice~\cite{nullspace-steering}. We use the same observation so that one prover covers every edit type.

\subsection*{Research gaps}
The literature leaves three gaps that we address (Table~\ref{tab:positioning}).
\begin{itemize}
	\item First, formal certificates in interpretability describe a fixed model: its accuracy, its circuits, or the stability of those circuits. The only certificate attached to an edit, BlockCert's, bounds output deviation on a finite prompt set. No prior work certifies the \emph{behavioral effect} of a mechanistic edit, removal of one skill and preservation of another, for every input in a continuous region (\S\ref{sec:method}, \S\ref{sec:findings}).
	\item Second, existing impossibility results concern future weight updates or oracle-free empirical certification. None establishes that finite deterministic black-box testing, adaptive or not, cannot certify removal (Proposition~\ref{prop:p3}).
	\item Third, the collateral damage of steering is documented empirically but has no exact account that separates surgical edits from steering (Proposition~\ref{prop:p4}).
\end{itemize}
We address all three gaps using the framework of \S\ref{sec:framework}: certificates of removal and preservation over continuous input regions, proved exactly by an SMT solver and extended by sound bound propagation where exact encoding stops (\S\ref{sec:method}), together with the impossibility and edit-separation results of Propositions~\ref{prop:p3} and~\ref{prop:p4}.

\begin{table*}[t]
	\centering
	\caption{Certification Object and Quantifier Domain comparison of related work with our own}
	\label{tab:positioning}
	\small
	\begin{tabular}{@{}p{0.24\textwidth}p{0.26\textwidth}p{0.24\textwidth}p{0.16\textwidth}@{}}
		\toprule
		\textbf{Work}                                    & \textbf{Certified object}                                     & \textbf{Quantifier domain}                               &
		\textbf{Edit's behavioral effect?}                                                                                                                                                                         \\
		\toprule
		Compact proofs~\cite{compact-proofs}             & accuracy bound of the unedited model                          & all inputs (small domains)                               & no                           \\
		\midrule
		Formal Mech. Interp.~\cite{hadad}                & circuit faithfulness under patching                           & continuous activation-patch values, finitely many inputs & no                           \\
		\midrule
		Certified circuits~\cite{certified-circuits}     & stability of a discovered circuit                             & dataset resamplings                                      & no                           \\
		\midrule
		Verifiable Transformers~\cite{somani}            & four properties of a circuit explanation                      & final-residual perturbations; per-input edge necessity   & no (disclaimed in~\cite{somani}, \S1.4)\\
		\midrule
		BlockCert~\cite{blockcert}                       & extraction fidelity; output-deviation bound                   & finite traced prompt set                                 & partial (how much, not what) \\
		\midrule
		Certified unlearning~\cite{certified-unlearning} & process indistinguishability                                  & distribution over models                                 & no (statistical)             \\
		\midrule
		Certified interventional fidelity~\cite{cif}     & statistical intervention effect                               & expectation over an input+intervention distribution      & no (average-case)            \\
		\midrule
		\textbf{This paper}                              & \textbf{removal + preservation + certified radius of an edit} & \textbf{continuous input regions and their inflations}   & \textbf{yes}                 \\
		\bottomrule
	\end{tabular}
\end{table*}

\section{Framework: certified edits and the impossibility of testing}
\label{sec:framework}

\subsection{Definitions}

\begin{definition}[model class]
	\label{def:model}
	\normalfont
	A \emph{network} is a function $f:\R^d\to\R^k$ of the form $f(x)=W_2\,\relu(W_1x+b_1)+b_2$, where $\relu$ is the componentwise ReLU, the hidden layer has \emph{width} $H$ (so $W_1\in\R^{H\times d}$, $b_1\in\R^{H}$, $W_2\in\R^{k\times H}$), and all weights are rational (every IEEE float is a rational, and the prover hands the solver the exact rationals). We write $f_h$ for output coordinate (``head'') $h\in\{1,\dots,k\}$ and, for each hidden neuron $j\in\{1,\dots,H\}$, the pre-activation of neuron $j$ is $\mathrm{pre}_j(x)=(W_1x+b_1)_j$. Every such $f$ is continuous and piecewise-linear; the input domain is $\D=[0,1]^d$.
\end{definition}

\begin{definition}[claim, certificate]
	\label{def:claim}
	\normalfont
	A \emph{claim} is a triplet of shape $(h,R,\rhd)$ with $R\subseteq\D$ a closed box and $\rhd\in\{>,\le\}$; it \emph{holds for} $g$ iff $\forall x\in R:\;g_h(x)\rhd 0$.  A \emph{certificate} is a proof of this statement, obtained by showing its negation $\exists x\in R:\lnot(g_h(x)\rhd 0)$ is unsatisfiable. \emph{Removal} of skill $A$ over $R_A$ is the claim $(A,R_A,\le)$ about the edited model (the unedited model satisfies $(A,R_A,>)$); \emph{preservation} of skill $B$ is a pair of claims on $B$'s regions.
\end{definition}

\begin{definition}[edit]
	\label{def:edit}
	\normalfont
	An \emph{edit} maps a network's weights to new weights of the same shapes. For a neuron set $C\subseteq\{1,\dots,H\}$ (the circuit), edit types are
	defined as follows:
	\begin{enumerate}
		\item \emph{Ablation}: For each neuron $j\in C$, zero the $j$th row of $W_1$ and the $j$th entry of $b_1$; this removes neuron $j$ from the hidden layer, so that $\mathrm{pre}_j(x)=0$.
		\item \emph{Weight edit}: For head $h$ and neuron $j$, zero the weight $W_2[h,j]$; this removes the contribution of neuron $j$ to head $h$.
		\item \emph{Steering}: For a steering vector $v\in\R^H$ (one entry per hidden neuron), add it to the hidden bias, $b_1\leftarrow b_1+v$; this shifts every neuron's pre-activation by a constant on all inputs, $\mathrm{pre}_j(x)\to \mathrm{pre}_j(x)+v_j$. The special case $v_j=-s$ for $j\in C$ and $v_j=0$ otherwise is \emph{targeted suppression at dose $s$ on $C$}: it pushes down exactly the neurons in $C$, by an amount $s$.
	\end{enumerate}
\end{definition}


\begin{definition}[inflation, certified radius]
	\label{def:radius}
	\normalfont
	For a region $R$ and $\eps\ge0$, the \emph{inflated region} is $R_\eps=\prod_i[\max(0,l_i-\eps),\min(1,u_i+\eps)]$ (so $R_0=R$, $R_\eps\subseteq\D$). Given a claim holding on $R$ for $g$, its \emph{certified radius} is $\epsstar=\sup\{\eps\in[0,\eps_{\max}]:(h,R_\eps,\rhd)\text{ holds for }g\}$, with $\eps_{\max}$ a fixed upper bound on the inflation considered, chosen large enough that $R_{\eps_{\max}}=\D$.
\end{definition}

\begin{remark}[the certified object and the float gap]
	\normalfont
	Certificates are statements about the exact rational network $f$ under real arithmetic. The executed program evaluates the same weights in float64 and can differ pointwise by rounding, bounded by a standard forward-error analysis ($\approx10^{-13}$ for our models). A certificate proved with slack $\gamma$ above that bound holds for the executed program too; a zero-margin claim need not transfer. The toy model's removal, preservation and control certificates are re-proved with slack $\gamma=10^{-9}$, and every threshold-gate transformer certificate carries slack $10^{-6}$ (\S\ref{sec:method}, App.~\ref{app:floatgap}).
\end{remark}

\subsection{Regions compose and the radius is computable (P1--P2)}

\begin{proposition}[region monotonicity and union closure]
	\label{prop:p1}
	\normalfont
	Given a model $g$, a head $h$, and sign $\rhd$,
	\begin{enumerate}[label=\textbf{(\alph*)}]
		\item If $(h,R,\rhd)$ holds and $R'\subseteq R$, then $(h,R',\rhd)$ holds.
		\item If $(h,R_i,\rhd)$ holds for every $i$ in any index set, then $(h,\bigcup_i R_i,\rhd)$ holds.
	\end{enumerate}
	The family of certified regions is thus downward closed and closed under arbitrary unions, with maximal element the truth set $\mathcal{T}=\{x\in\D:g_h(x)\rhd 0\}$: a region is certifiable iff it is a subset of $\mathcal{T}$.
\end{proposition}

Both parts are immediate from the meaning of the universal quantifier (proof in App.~\ref{app:proofs}). Proposition 1 licenses reporting a maximal certified region and gluing per-patch certificates without re-proving, and underlies the monotonicity of the radius in $\eps$; it also lets removal and preservation be certified independently and combined.

\begin{proposition}[the certified radius is well-defined and cheap]
	\label{prop:p2}
	\normalfont
	Let $g$ be a network, $R$ a box, and suppose $(h,R,\rhd)$ holds. Let $S=\{\eps\in[0,\eps_{\max}]:(h,R_\eps,\rhd)\text{ holds}\}$, $\epsstar=\sup S$.  Then:
	\begin{enumerate}[label=\textbf{(\alph*)}]
		\item $S$ is an interval containing $0$;
		\item if $\rhd$ is $\le$, then the supremum is \emph{attained} ($\epsstar\in S$); while for $>$, only $[0,\epsstar)\subseteq S$ is guaranteed;
		\item for each rational $\eps$, membership $\eps\in S$ is decided exactly by the (un)satisfiability of a quantifier-free linear-real-arithmetic formula, for which the solver is sound and complete;
		\item bisection returns, in at most $2+\lceil\log_2(\eps_{\max}/\tau)\rceil$ queries, either a counterexample, ``saturated'', or a bracket $[\ell,\mathrm{hi}]$ with $\ell\in S$, $\mathrm{hi}\notin S$, $\mathrm{hi}-\ell\le\tau$, hence $|\ell-\epsstar|\le\tau$.
	\end{enumerate}
\end{proposition}

The attainment asymmetry in (b) matters for reporting: \emph{removal} radii (sign $\le$) are genuine maxima, whereas \emph{preservation} radii (sign $>$) are suprema our procedure brackets within $\tau$. The bracket guarantee assumes every query is answered; the code treats a solver timeout conservatively as a failed probe, which preserves the soundness of $\ell\in S$ but voids the bracket (no query returned ``unknown'' in any experiment reported here). Full proof in App.~\ref{app:proofs}.

\subsection{No finite test certifies removal (P3)}

This is the central negative result: the intervention illusion is not an artifact of particular models, but a theorem of the framework itself.

\begin{proposition}[testing cannot certify removal]
	\label{prop:p3}
	\normalfont
	Given input dimension $d\ge1$, a removal region $R=\prod_i[l_i,u_i]\subseteq\D$ with $l_0<u_0$, any finite test set $T\subset\D$, and any $\delta>0$, there exist a network $f$, a neuron set $C$, and the edited model $g=\ablate(f,C)$ such that:
	\begin{enumerate}[label=\textbf{(\alph*)}]
		\item the unedited $f$ satisfies the full skill specification;
		\item $g_A(t)\le0$ for every $t\in T\cap R$ \emph{and} both preservation claims for skill $B$ hold for $g$ on their entire regions (so even a solver checking preservation approves the edit);
		\item the survivor set $S_g=\{x\in R:\,g_A(x)>0\}$ is nonempty and open, and its relative volume satisfies $\vol(S_g)/\vol(R)<\delta$.
	\end{enumerate}
\end{proposition}

\begin{corollary}[adaptive protocols]
	\label{cor:c32}
	\normalfont
	Let $P$ be any deterministic protocol that queries an edited model at finitely many inputs where each is possibly chosen from earlier answers and then gives its verdicts. If $P$ approves some genuinely edited model, it also approves some edited model that retains the skill on a positive-measure set. Hence, no finite-query deterministic protocol is both \emph{sound} for removal and \emph{nontrivial} over any region with nonempty interior.
\end{corollary}

\emph{Construction (sketch; full proof App.~\ref{app:proofs}).} The finitely many test points have finitely many first coordinates; with $l_0,u_0$ they leave an open interval $(c-w,c+w)$ disjoint from all of them. Three hidden neurons with pre-activations $x_0-(c-w),\,x_0-c,\,x_0-(c+w)$ feed head $A$ the combination
\[
	\mathrm{tent}_\beta(x_0)=\beta\bigl[\relu(x_0-(c{-}w))-2\relu(x_0-c)+\relu(x_0-(c{+}w))\bigr],
\]
which is a ``hat'' that rises to peak $\beta w$ at $x_0=c$, returns to zero at $c+w$, and is \emph{identically zero} thereafter (the slopes $\beta,-2\beta,\beta$ cancel); in particular it vanishes at every test point. Lowering head $A$'s bias by $\beta w/2$ makes $g_A=\mathrm{tent}_\beta-\beta w/2$ report ``removed'' at every test point yet exceed $0$ exactly on the slab $|x_0-c|<w/2$, of relative volume $w/(u_0-l_0)<\delta$ (Fig.~\ref{fig:illusion}b). For Corollary~\ref{cor:c32}, adding the same gadget to a genuinely edited $g_0$ leaves its transcript on $P$'s queries unchanged, and therefore $P$'s deterministic verdict is unchanged, while re-introducing a positive-measure survivor.

The scope is deliberate and stated: Cor.~\ref{cor:c32} covers deterministic black-box protocols (adaptive included); a randomized adaptive tester and a white-box analysis (such as the SMT solver used in this paper) are outside it. Informal statements of P3 should therefore always be read as ``no finite deterministic black-box behavioral test.'' The constructive proof is the recipe executed by our illusion experiment, and \S\ref{sec:findings} reports naturally trained illusions showing the hypothesis needs no adversarial constructor.

\subsection{Surgery vs.\ steering (P4)}

\begin{proposition}[edit-class separation, one hidden layer]
	\label{prop:p4}
	\normalfont
	Let $g$ be one-hidden-layer, $C\subseteq\{1,\dots,H\}$ a neuron set, $\mathrm{steer}_s$ targeted suppression at dose $s$, $\mathrm{abl}$ ablation of $C$, and $M_j(R)=\max_{x\in R}\mathrm{pre}_j(x)$ (attained at a box corner). Then:
	\begin{enumerate}[label=\textbf{(\alph*)}]
		\item $\mathrm{steer}_{s,h}(x)=\mathrm{abl}_h(x)+\sum_{j\in C}W_2[h,j]\,\relu(\mathrm{pre}_j(x)-s)$ for all $x$.
		\item If $s\ge\max_{j\in C}M_j(\D)$ then $\mathrm{steer}_s\equiv\mathrm{abl}$ on $\D$, so every certificate and radius is identical.
		\item Suppose $\mathrm{abl}$ certifies removal over $R$, and let $S(R)$ be the set of doses $s$ at which $\mathrm{steer}_s$ also certifies it. Then $S(R)$ is closed, contains $[\max_{j\in C}M_j(\D),\infty)$, and shrinks as $R$ grows, so the minimal certifying dose $s^*(R)=\min S(R)$ is attained and non-decreasing in $R$; hence $s^*(R_\eps)$ is non-decreasing in $\eps$ (more wiggle room demands more dose). If moreover $W_2[A,j]\ge0$ for every $j\in C$, then $S(R)=[s^*(R),\infty)$ is an interval. Without this sign condition it need not be one: a larger dose can undo a removal that a smaller dose certified (App.~\ref{app:proofs}).
		\item Let $v=s\hat v$ be a general steering vector with dose $s\ge0$. For an input $x$, let $\mathcal A(x)$ be the set of neurons active at $x$ before the edit, $\kappa(x)=\sum_{j\in\mathcal A(x)}W_2[B,j]\hat v_j$, and $\bar s(x)\in[0,\infty]$ the infimum of the doses at which some neuron with $\hat v_j\neq0$ has a different activation status at $x$ than before the edit. For $0\le s\le\bar s(x)$, head $B$'s logit at $x$ shifts by exactly $s\,\kappa(x)$: linear in dose up to the first activation flip, and continuous and piecewise linear in $s$ beyond it. If $W_2[B,j]=0$ for every $j$ with $\hat v_j\neq0$, head $B$ is unchanged at every dose and every input (zero collateral). If $\kappa(x)<0$ at some $x\in R$, preservation $(B,R,>)$ is refuted at every dose $s$ with $g_B(x)/|\kappa(x)|\le s\le\bar s(x)$; the formula guarantees a break only when this range is nonempty.
	\end{enumerate}
\end{proposition}

P4 turns the edit comparison of \S\ref{sec:findings-edits} from measurement into theorem for the one-hidden-layer class: surgical edits \emph{must} show zero collateral when the circuit is disjoint from head $B$; steering heard by head $B$ shifts $B$'s logit linearly in dose at each input until that input's first activation flip, and refutes preservation at any dose that drives such an input's logit to zero. The linearity is per input and per dose range. Proposition~\ref{prop:p4}(d) says nothing directly about the certified radius; the near-linear decline of the radius in Fig.~\ref{fig:edits}b is measured. The sign condition in (c) holds in every one-hidden-layer targeted-suppression experiment we report: each steered circuit there is a single neuron that feeds head $A$ with a positive weight. The deeper-network generalization is open (the leftover term propagates through later nonlinearities). P2--P4 each have a machine-checked
witness (App.~\ref{app:witnesses}); P1 follows from the meaning of the universal quantifier.

\section{Method}
\label{sec:method}

\subsection{Exact encoding (SMT)}
\label{sec:method-exact}

We encode a network (Def.~\ref{def:model}) into the Z3 SMT solver~\cite{z3} using \emph{exact rationals}: every weight is converted to a fraction, and each ReLU $\relu(z)$ becomes the split case $\mathrm{ite}(z>0,z,0)$.  A claim $(h,R,\rhd)$ is proved by asserting its negation (an input in $R$ violating it) and obtaining \texttt{unsat}. Because the formula is quantifier-free linear real arithmetic (QF-LRA) over rationals, the verdict is exact and complete (Proposition~\ref{prop:p2}c). An \texttt{unsat} certifies the claim for \emph{every} point of $R$, whereas \texttt{sat} returns a concrete counterexample (the solver catching an input for which the skill still fires; such an input might be missed by a test suite). Certified radii are found by bisecting proofs (Proposition~\ref{prop:p2}d). Every proof is cross-checked against a brute-force grid; because the grid runs the float program while the solver reasons about the ideal network, disagreement beyond float noise signals a bug, and disagreement within $\sim\!10^{-12}$ of zero signals a genuinely zero-margin claim (App.~\ref{app:floatgap}).

\paragraph{Sequences and continuous noise in one solver query}
For a token model we quantify in a single query over all token sequences $\times$ a continuous ball of embedding noise. Two encodings are available. The Boolean encoding uses Boolean token selectors (the exact discrete claim) but branches combinatorially and exceeds the time budget at every size we tried. The \emph{hull relaxation} replaces the discrete token choice with continuous mixture weights over the embedding simplex; the resulting region is a strict superset of the discrete inputs, so an \texttt{unsat} still certifies the discrete claim via Proposition~\ref{prop:p1}a, with no Boolean branching. The hull relaxation is what makes the transformer tractable, once the model is shrunk to the measured frontier (\S\ref{sec:findings}).

\paragraph{Handling of the float gap}
Certificates concern the exact rational network while the executed model runs float64. A standard forward-error bound puts the gap at $\approx10^{-13}$ here. The toy model's removal, preservation and control certificates are re-proved with explicit slack $\gamma=10^{-9}$, and every threshold-gate transformer certificate carries slack $10^{-6}$, so these certificates transfer to the executed program. A constructed zero-margin claim is proved by the solver yet violated by the float program by $\sim\!10^{-16}$ (App.~\ref{app:floatgap}).

\subsection{Bound-propagation extension (CROWN)}
\label{sec:method-bp}

Softmax and LayerNorm cannot be encoded exactly: their exponentiations and divisions are outside linear arithmetic. To certify edits on a standard softmax\,+\,LayerNorm transformer we switch to bound propagation (auto\_LiRPA \cite{autolirpa} / CROWN~\cite{crown}), which computes sound linear lower and upper bounds on the output over an input ball. Bound propagation is \emph{sound but incomplete}: it may fail to certify a true claim, but never certifies a false one. Two checks support its results.

\paragraph{M0 (validate the pipeline against the exact solver)}
Bound propagation is sound by construction~\cite{crown}, so this check (M0) validates the implementation rather than the guarantee. On the toy ReLU MLP both tools apply and the exact solver is complete (Proposition~\ref{prop:p2}c), so its certified radius is the true radius. CROWN reproduces it exactly (\S\ref{sec:findings}): the two pipelines encode the same model, edit and claim, and the relaxation is tight on this subject. Tightness does not transfer to the softmax subject, where no exact oracle is available; attack bracketing covers that case.

\paragraph{Attack bracketing}
Beyond the exact frontier we bracket every certified radius from above by a projected-gradient attack: for each claim we report $\text{certified}\le\text{true}\le\text{attack}$. A certified radius above an attack radius would be unsound; no reported radius is. The gap between the two measures prover looseness, reported per claim.

\subsection{Scope of the bound-propagation results}
\label{sec:scope-bp}
A CROWN radius is a lower bound: the interval up to the attack radius is unresolved, and its width measures prover looseness rather than the edit's fragility. The exact solver instead returns the tipping point itself (Proposition~\ref{prop:p2}b). The relaxation is degenerate at $\eps=0$, so the bisection probes only balls of radius $\ge10^{-4}$, which is the smallest radius we report. Greedy component search finds no separable head/MLP circuit in a from-scratch softmax transformer, so skill $A$ is removed by a position-scoped attention knockout. The certificate concerns the behavioral effect of that intervention instead of the identification of a component-level circuit.

\section{Findings}
\label{sec:findings}

The following findings are discussed in order of safety significance and each finding is supported by whichever experimental subject demonstrates it. The subjects themselves (a progression from a two-input ReLU network to a softmax\,+\,LayerNorm transformer) are summarized, with every certified radius, in Table~\ref{tab:results} and detailed in App.~\ref{app:repro}.

\begin{table*}[t]
	\centering
	\caption{Certified radii $\epsstar$ (Def.~\ref{def:radius}) by subject. Parenthesized values are skill $B$'s radius before the edit, so collateral reads as the drop. Exact removal radii are attained tipping points and exact preservation radii are bracketed within the bisection tolerance (Proposition~\ref{prop:p2}b,d); bound-propagation radii are sound lower bounds. Values from committed reports (App.~\ref{app:repro}).}
	\label{tab:results}
	\footnotesize
	\setlength{\tabcolsep}{4pt}
	\begin{tabular}{@{}lllrll@{}}
		\toprule
		\textbf{Subject}         & \textbf{Size}                   & \textbf{Prover} & \textbf{Pert.\ dims} & \textbf{Removal $\epsstar$} & \textbf{Preservation $\epsstar$ (unedited)} \\
		\midrule
		Toy MLP (separated)      & $d{=}2$, 16 hidden              & SMT             & 2                    & $\ge0.6$\,$^{a}$            & $0.0996$ $(0.0996)$                         \\
		Toy MLP (entangled)      & $d{=}2$, 16 hidden              & SMT             & 2                    & $\ge0.6$\,$^{a}$            & $0.0961$ $(0.0961)$                         \\
		Deep MLP                 & $3\text{-}12\text{-}8\text{-}2$ & SMT             & 3                    & $\ge0.5$\,$^{a}$            & $0.0742$ $(0.0938)$\,$^{b}$                 \\
		Gate transformer         & 8 wide, 2 heads                 & SMT (hull)      & 48                   & $0.0156$                    & $0.0125$ $(0.0344)$                         \\
		Adder transformer        & $p{=}5$, 8 wide, 2 heads        & SMT (siamese)   & 40                   & $\ge0.05$\,$^{a}$           & $\ge0.002$\,$^{c}$                          \\
		Softmax\,+\,LN transf.\  & 64 wide, 4 heads, 2 layers      & CROWN           & \textbf{448}         & $0.0091$\,$^{d}$            & $0.0079$, $0.0098$\,$^{d}$                  \\
		\bottomrule
	\end{tabular}

	\parbox{\textwidth}{\footnotesize\vspace{2pt}
		$^{a}$ saturated at the search ceiling: bounded by the probe range, not the model.
		$^{b}$ four-neuron layer-1 circuit, bottleneck layer banned.
		$^{c}$ plus exact correctness on all $625$ clean sequences at $\eps{=}0$.
		$^{d}$ attack-bracketed: $0.0091\!\le\!0.03$, $0.0079\!\le\!0.05$, $0.0098\!\le\!0.1$.}
\end{table*}

\subsection{An edit's behavioral effect (removal and preservation of skills) can be proved over a whole input region on a transformer}
\label{sec:findings-positive}

On a standard softmax\,+\,LayerNorm transformer (which the SMT solver cannot encode) bound propagation certifies that a position-scoped edit removes skill $A$ (accuracy drops to chance) and preserves skill $B$ (accuracy $100\%$), over continuous embedding-noise balls in $448$ perturbation dimensions, roughly $9\times$ the exact frontier established by SMT. Every radius is bracketed by attack ($\text{certified}\le\text{true}\le\text{attack}$): removal $0.0091\le0.03$, preservation $0.0079\le0.05$ and $0.0098\le0.1$; none unsound (Fig.~\ref{fig:bp}).

The same guarantee holds exactly on a threshold-gate transformer. The exact solver certifies removal and preservation over all token sequences $\times$ a continuous embedding ball in one query. The ablated circuit spans an attention head and three MLP neurons, and switching off the heads alone, or the MLP alone, leaves skill $A$ firing (Fig.~\ref{fig:circuit}). Removal certifies to radius $0.016$ and preservation to $0.013$ (Table~\ref{tab:results}), with the collateral quantified: skill $B$'s radius falls from $0.034$ to $0.013$ under the edit. On a known-formula subject (two independent modular adders on disjoint positions) both sides are certified over noise: exact-rational correctness on all $625$ clean sequences, removal to radius $\ge0.05$ and preservation robust to radius $\ge0.002$ (App.~\ref{app:repro}).

\begin{figure}[t]
	\centering
	\includegraphics[width=\linewidth]{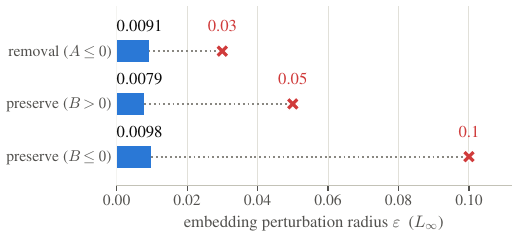}
	\caption{Certified removal and preservation on a standard softmax\,+\,LayerNorm transformer, which the exact encoding cannot represent, over embedding-noise balls in $448$ dimensions (about $9\times$ the exact frontier of Table~\ref{tab:frontier}). Solid bars are sound certified radii from bound propagation; each cross marks the radius at which a projected-gradient attack first breaks the claim, so the true radius lies on the dotted span between them. No attack succeeds below a certified radius.}
	\label{fig:bp}
\end{figure}

\begin{figure}[t]
	\centering
	\includegraphics[width=\linewidth]{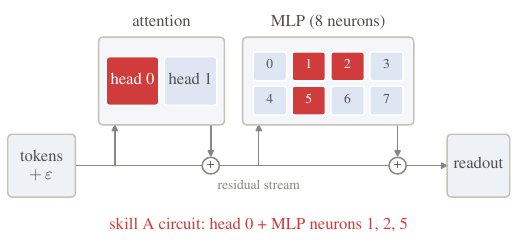}
	\caption{The certified circuit is distributed across attention and the MLP. One decoder block of the exact subject ($8$ wide, $2$ heads, $L=6$), drawn as its forward pass runs: each block reads the residual stream and adds its output back. Skill~$A$'s circuit (filled) was found by a greedy search over heads and MLP neurons that preserves skill~$B$; switching off the heads alone, or the MLP alone, leaves skill~$A$ firing. Ablating the whole circuit certifies removal of~$A$ and preservation of~$B$ (Table~\ref{tab:results}).}
	\label{fig:circuit}
\end{figure}

\subsection{No finite test protocol can certify removal}
\label{sec:findings-illusion}

Proposition~\ref{prop:p3} states that the guarantee is unreachable by testing. As witness, we show a constructed edit passes a $40{,}401$-point grid with preservation of skill $B$ proved, yet the solver refutes removal with a concrete surviving input. Ordinary training through our experiments produces three naturally occurring illusions in $160$ models (from models with $\ge 4$ inputs), each passing a full grid plus $1{,}000$ random inputs plus preservation checks, and each is refuted by the solver. Their survivor pockets, where skill $A$ is still alive, are tiny: zero hits in $2\times10^6$ uniform darts give a $95\%$ Clopper--Pearson ceiling of $1.5\times10^{-6}$ of the domain. A diligent tester may ship these edits believing the skill removed, while the skill survives in a vanishing-but-nonzero pocket (Fig.~\ref{fig:illusion}).

\subsection{Feature non-interference}
\label{sec:findings-influence}

A stronger security guarantee is demonstrating that the component's output is provably independent of the forbidden input rather than merely that the readout is below a threshold. We certify it directly with a two-copy (siamese) encoding (Fig.~\ref{fig:influence}) that bounds a \emph{certified influence}: the worst-case change in the head's output as the forbidden input varies over the whole domain. On the transformer the certified influence reaches \textbf{exactly $0$}, meaning the readout ignores the quote content everywhere, out to a \emph{larger} radius than removal itself. On the toy subject the surgical edit drives influence from $22$ to $<\!0.001$; on a deliberately entangled ``messy'' model the edit certifies removal yet influence falls only from $63$ to $10$, quantifying that \emph{removal is not deafness}. The illusion edit is refuted here too, with a concrete input \emph{pair} and influence ceiling $0.60$---exactly the tent's height.

\begin{figure}[t]
	\centering
	\includegraphics[width=\linewidth]{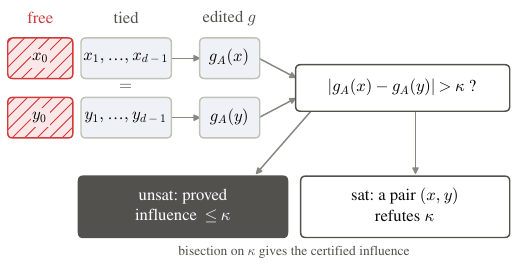}
	\caption{The certified-influence query. The edited network~$g$ is encoded twice, over inputs $x$ and $y$ in the same region that agree on every coordinate except the forbidden one ($x_0$ on the toy models; the quote-token content on the transformer, with the embedding noise shared). The solver searches for a pair whose skill-$A$ outputs differ by more than~$\kappa$: \texttt{unsat} proves that changing the forbidden input alone moves the output by at most~$\kappa$, and \texttt{sat} returns such a pair. Bisection on~$\kappa$ yields the certified influence; influence~$0$ certifies feature non-interference.}
	\label{fig:influence}
\end{figure}

\subsection{Removal is robust to input perturbation, up to a certified radius}
\label{sec:findings-radius}

Removal and preservation hold not only on the claimed region but on its inflation by the certified radius (Def.~\ref{def:radius}). Read on the region's edge, the radius measures tightness (Proposition~\ref{prop:p2}): the inflated edge stops where the edited model's own decision boundary does, not at a margin added by the prover (Fig.~\ref{fig:radius}). Across the correctness claims of both toy models, certified edges reach within $0.0012$ of the $0.5$ rule for skill~$A$ and within $0.0039$ at worst. The residual gap is the model's decision margin, not prover slack.

\begin{figure}[t]
	\centering
	\includegraphics[width=\linewidth]{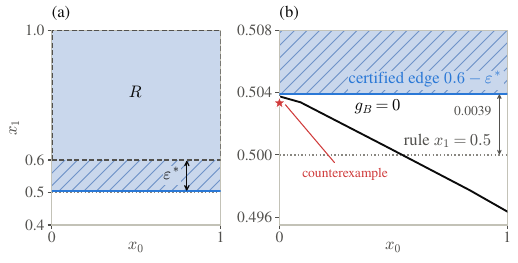}
	\caption{The certified radius in input space, for the entangled toy model after ablation and the claim that skill~$B$ fires on $R=[0,1]\times[0.6,1]$ (Table~\ref{tab:results}). (a)~$R$ (dashed) and its inflation by the certified radius $\epsstar=0.0961$: the certified region is shaded, and the band that inflation adds to~$R$ is hatched. (b)~The strip near the rule $x_1=0.5$ (dotted). The certified edge stops $0.0039$ above the rule, within the bisection tolerance of the highest point of the edited model's own decision boundary $g_B=0$ (solid), at $x_0=0$. At the smallest refuted inflation the solver returns a counterexample (star) beyond that boundary.}
	\label{fig:radius}
\end{figure}

\begin{figure*}[t]
	\centering
	\includegraphics[width=\textwidth]{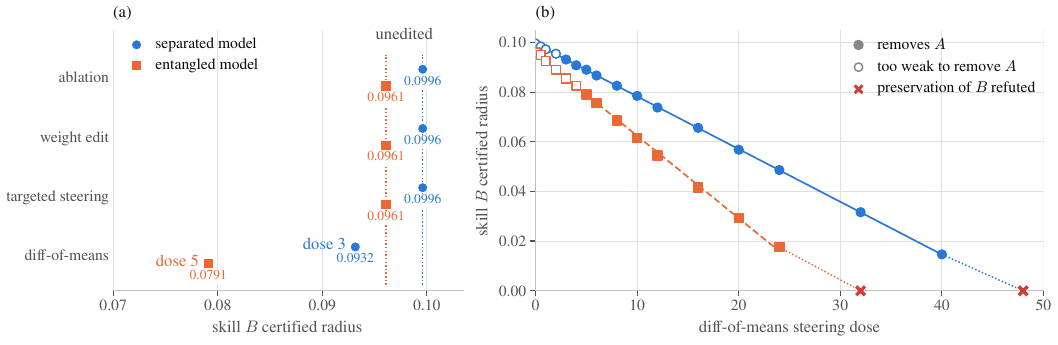}
	\caption{Surgical edits certify zero collateral; diff-of-means steering certifiably erodes the preserved skill. (a)~Skill~$B$'s certified radius after each edit type on a separated and an entangled toy model, against each model's unedited radius (dotted). Ablation, weight editing and targeted steering leave it unchanged; diff-of-means falls short even at the smallest dose that removes skill~$A$. (b)~The same radius against diff-of-means dose. Each marker is a separate proof: filled markers also remove skill~$A$, open ones are too weak to. The radius falls near-linearly in dose (least-squares fits, $R^2\ge0.999$) until preservation is refuted (cross). This trend is measured; Proposition~\ref{prop:p4}(d) concerns the logit shift, not the radius.}
	\label{fig:edits}
\end{figure*}

\subsection{Surgical edits certify collateral-free removal, unlike steering}
\label{sec:findings-edits}

Proposition~\ref{prop:p4} predicts, and the experiments confirm, a sharp split.  Surgical edits (ablation, weight edit) certify removal at maximal radius with provably zero (disjoint circuit) or bounded collateral. Realistic diff-of-means steering provably erodes the preserved skill's margin as dose rises (on the separated model the certified radius falls from $0.0996$ to $0.0908$ at dose $4$ and to $0.0656$ at dose $16$), until it breaks preservation outright (Fig.~\ref{fig:edits}). On the transformer the split is starkest: no dose both removes $A$ and preserves $B$ (App.~\ref{app:repro}). A sparse-autoencoder feature clamp, on a model where skill $A$ splits across three features, is certified by the same prover: removal to radius $\ge0.5$ and preservation to $0.064$, against ablation's $0.100$ on the same model (App.~\ref{app:repro}). P4(d) proves that the logit shift is linear in dose at each input up to that input's first activation flip; the decline of the certified radius over the full dose sweep is measured (Fig.~\ref{fig:edits}b).

\subsection{The exact verification frontier and its sound extension}
\label{sec:findings-reach}

The exact frontier is set by internal branching: the Boolean encoding's cost is exponential in the model's case splits, and it exceeds the time budget at every size. The shrink-plus-hull-relaxation path reaches a two-head transformer exactly (Table~\ref{tab:frontier}). Bound propagation goes past it, validated to \emph{equal} the exact radius on every claim where both provers apply---$0.5{=}0.5$ for removal and $0.0996{=}0.0996$ for both preservation claims (M0)---and attack-bracketed beyond (M1). Finally, the continuous radius is not a disguised discrete-robustness claim. The smallest single-token swap moves the embedding by $2.94$, some $\approx\!324\times$ the certified ball, so the continuous guarantee and discrete-prompt robustness are different objects (App.~\ref{app:repro}).

Taken together, the findings answer the three gaps of \S\ref{sec:related}. First, an edit's effect can be proved rather than sampled: on every subject, from a two-input network to a softmax\,+\,LayerNorm transformer, removal and preservation hold for every input in a continuous region, and on the toy models the exact certificates stop at the model's own decision boundary, not at slack in the prover. Second, passing a test suite is not evidence of removal: ordinary training produces edits that pass every test while the skill survives, and the solver refutes them with a concrete input. Third, the choice of edit is itself certifiable: surgical edits remove a skill with provably zero collateral, whereas steering trades removal against collateral in a way no dose resolves. For a practitioner validating a safety edit, the certificate therefore replaces the test set as the evidence that the edit did what it was meant to do, within the scale bound that \S\ref{sec:limitations} states.
\begin{table}[t]
	\centering
	\caption{Where exact certification stops. Each row is a configuration of the threshold-gate transformer; each cell is one all-sequences hull query of the unedited model's skill-$A$ claim, without noise or with it. The rule marks the frontier; the exact subject of Table~\ref{tab:results} has the bold row's size, with two heads.}
	\label{tab:frontier}
	\footnotesize
	\setlength{\tabcolsep}{4pt}
	\begin{tabular}{@{}lrll@{}}
		\toprule
		\textbf{Configuration}         & \textbf{Noise vars} & $\varepsilon=0$    & $\varepsilon=0.005$ \\
		\midrule
		8 wide, 1 head, $L=4$          & 32                  & \checkmark\ 32\,s  & \checkmark\ 155\,s  \\
		\textbf{8 wide, 1 head, $L=6$} & \textbf{48}         & \checkmark\ 40\,s  & \checkmark\ 230\,s  \\
		\midrule
		8 wide, 1 head, $L=8$          & 64                  & \emph{too loose}   & \emph{too loose}    \\
		12 wide, 1 head, $L=6$         & 72                  & \checkmark\ 31\,s  & \emph{timeout}      \\
		16 wide, 1 head, $L=6$         & 96                  & \checkmark\ 120\,s & \emph{timeout}      \\
		16 wide, 2 heads, $L=8$        & 128                 & \emph{timeout}     & \emph{timeout}      \\
		\bottomrule
	\end{tabular}

	\parbox{\columnwidth}{\footnotesize\vspace{2pt}
		\checkmark\ proved, with solver time; \emph{too loose}: the hull relaxation is inconclusive; \emph{timeout}: exceeds the $300$\,s query budget, and the $128$-variable row also exceeds $30$\,min at $\varepsilon\in\{0.01,0.05\}$.}
\end{table}

\section{Limitations and Scope}
\label{sec:limitations}

\textbf{Scale:} Guarantees are proved on small,
standard-architecture networks; the exact frontier is exponential in internal branching (\S\ref{sec:findings-reach}), and bound-propagation results, while sound, are not exact. Wide certified-to-attack gaps reflect prover looseness, and radii are reported over balls of radius $\ge10^{-4}$ (CROWN is degenerate at $\eps=0$). The edit on the softmax model is a position path-patch, not a component circuit, because a from-scratch softmax transformer yields no separable head/MLP circuit.

\textbf{Input region instead of arbitrary prompt:} The guarantee is over a class of token sequences $\times$ a continuous embedding ball. In the exact cases this covers every discrete prompt in the model's input space and a continuous neighborhood of each (a superset of all possible prompts). But the continuous axis is embedding-space noise. A certified radius is not a discrete-prompt-edit distance, and indeed a single token swap is $\approx\!324\times$ the certified ball (\S\ref{sec:findings-reach}). We therefore claim continuous-region robustness, not frontier-scale prompt robustness.

\textbf{The decidable-oracle assumption:} Every skill we certify has a decidable ground truth, a computable function such as ``input contains more \guillemotleft{} than \guillemotright{}'' or $(a+b)\bmod p$, which is exactly what makes ``removed'' or ``preserved'' a crisp, checkable claim. Real harmful capabilities in frontier models have no such decidable oracle: whether an output ``exhibits a harmful skill'' is contested and, in general, undecidable. The certified-removal setup therefore presupposes a decidable specification of the target skill, and that assumption does not transfer to real harms for free. The known-formula subject is a deliberate step from a threshold rule to an arithmetic specification; the future-work counterpart is in \S\ref{sec:conclusion}. The demonstration is therefore a proof of concept, not a certificate over a real dangerous capability.

\section{Conclusion and Future Work}
\label{sec:conclusion}

We introduced \emph{certified mechanistic edits}: proofs, holding over a continuous input region, that an edit removes one skill and preserves another. We (1) built the framework (P1--P4) and proved that no finite deterministic black-box test can deliver such a guarantee (P3), (2) gave an exact method and mapped its frontier, and (3) went past that frontier soundly with bound propagation on a standard softmax\,+\,LayerNorm transformer.

The natural next step is \emph{certified mechanistic edits at scale and on real skills}: bound propagation on a pretrained small language model with a genuinely safety-relevant capability and a real edit (SAE clamp or unlearning), and a discrete edit-distance certificate to complement the continuous one. The prerequisite that our limitations expose is sharper: applying any of this to real harms first requires a framework that decides (deterministically or probabilistically) whether a model output exhibits a target capability. That decidable or estimable oracle is what would let the certified-edit machinery leave the toy and known-formula regime. Such a step would show that mechanistic edits can be certified on real models and real skills, and that the certified radius is meaningful in discrete prompt space.


\section*{Open Science}
\addcontentsline{toc}{section}{Open Science} All claims are reproducible from committed code and fixed seeds: each reported number is regenerated by a single command from a versioned report (App.~\ref{app:repro}). The exact pipeline depends only on \texttt{numpy} and an SMT solver; the bound-propagation extension isolates its dependencies in a separate environment. The code, the committed reports and the scripts that regenerate the figures are publicly available at \url{https://github.com/manchitro/certified-mech-edits}.

\section*{LLM usage considerations}
\addcontentsline{toc}{section}{LLM usage considerations}
LLMs were used for editorial purposes in this manuscript, and all outputs were inspected by the authors to ensure accuracy and originality. They were also used as coding and writing assistants (code scaffolding, prose editing, and literature-scan support). All formal statements, proofs, experimental designs, and numerical results were produced and verified by the authors and by the machine-checked witnesses of App.~\ref{app:witnesses}; no experimental result or citation was taken from an LLM without independent verification against primary sources.

\section*{Ethical Considerations}
\addcontentsline{toc}{section}{Ethical Considerations} This work develops \emph{defensive} guarantees: proofs that a safety edit did what it was intended to do. Two dual-use points need stating. First, Proposition~\ref{prop:p3} and the intervention-illusion experiments show that an edit can pass exhaustive testing while a skill survives; we present this to caution practitioners against over-trusting test-based validation, not as an evasion recipe; the survivors are exhibited precisely so they can be caught. Second, Proposition~\ref{prop:p4} and \S\ref{sec:findings-edits} characterize when steering damages an unrelated skill; the same algebra that predicts collateral also bounds it, and we report it to make edits auditable. All models and skills studied here are toy or known-formula constructs with no dangerous capability; no human-subjects data, personal data, or deployed system is
involved.

\bibliographystyle{IEEEtran}
\bibliography{refs}

\begin{thebibliography}{10}
\providecommand{\url}[1]{#1}
\csname url@samestyle\endcsname
\providecommand{\newblock}{\relax}
\providecommand{\bibinfo}[2]{#2}
\providecommand{\BIBentrySTDinterwordspacing}{\spaceskip=0pt\relax}
\providecommand{\BIBentryALTinterwordstretchfactor}{4}
\providecommand{\BIBentryALTinterwordspacing}{\spaceskip=\fontdimen2\font plus
\BIBentryALTinterwordstretchfactor\fontdimen3\font minus
  \fontdimen4\font\relax}
\providecommand{\BIBforeignlanguage}[2]{{%
\expandafter\ifx\csname l@#1\endcsname\relax
\typeout{** WARNING: IEEEtran.bst: No hyphenation pattern has been}%
\typeout{** loaded for the language `#1'. Using the pattern for}%
\typeout{** the default language instead.}%
\else
\language=\csname l@#1\endcsname
\fi
#2}}
\providecommand{\BIBdecl}{\relax}
\BIBdecl

\bibitem{wei-jailbroken}
A.~Wei, N.~Haghtalab, and J.~Steinhardt, ``Jailbroken: How does llm safety
  training fail?'' \emph{Advances in neural information processing systems},
  vol.~36, pp. 80\,079--80\,110, 2023.

\bibitem{zou-gcg}
A.~Zou, Z.~Wang, N.~Carlini, M.~Nasr, J.~Z. Kolter, and M.~Fredrikson,
  ``Universal and transferable adversarial attacks on aligned language
  models,'' \emph{arXiv preprint arXiv:2307.15043}, 2023.

\bibitem{lucki-unlearning}
\BIBentryALTinterwordspacing
J.~Łucki, B.~Wei, Y.~Huang, P.~Henderson, F.~Tramèr, and J.~Rando, ``An
  adversarial perspective on machine unlearning for ai safety,'' 2025.
  [Online]. Available: \url{https://arxiv.org/abs/2409.18025}
\BIBentrySTDinterwordspacing

\bibitem{compact-proofs}
J.~Gross, R.~Agrawal, T.~Kwa, E.~Ong, C.~H. Yip, A.~Gibson, S.~Noubir, and
  L.~Chan, ``Compact proofs of model performance via mechanistic
  interpretability,'' \emph{Advances in Neural Information Processing Systems},
  vol.~37, pp. 79\,453--79\,515, 2024.

\bibitem{hadad}
I.~Hadad, G.~Katz, and S.~Bassan, ``Formal mechanistic interpretability:
  Automated circuit discovery with provable guarantees,'' \emph{arXiv preprint
  arXiv:2602.16823}, 2026.

\bibitem{somani}
N.~Somani, ``Towards verifiable transformers: Solver-checkable circuit
  explanations,'' \emph{arXiv preprint arXiv:2605.24033}, 2026.

\bibitem{certified-circuits}
A.~Anani, T.~Lorenz, B.~Schiele, M.~Fritz, and J.~Fischer, ``Certified
  circuits: Stability guarantees for mechanistic circuits,'' \emph{arXiv
  preprint arXiv:2602.22968}, 2026.

\bibitem{certified-unlearning}
A.~Koloskova, Y.~Allouah, A.~Jha, R.~Guerraoui, and S.~Koyejo, ``Certified
  unlearning for neural networks,'' \emph{arXiv preprint arXiv:2506.06985},
  2025.

\bibitem{blockcert}
S.~Andric, ``{BlockCert}: Certified blockwise extraction of transformer
  mechanisms,'' \emph{arXiv preprint arXiv:2511.17645}, 2025.

\bibitem{cif}
A.~Asiaee, ``Certified interventional fidelity: Anytime-valid, adaptive
  evaluation of causal claims in mechanistic interpretability,'' \emph{arXiv
  preprint arXiv:2607.08349}, 2026.

\bibitem{stat-unlearn}
A.~Pandey and S.~Kulkarni, ``Statistical unlearning of distributions: A
  hypothesis testing approach,'' \emph{arXiv preprint arXiv:2605.16645}, 2026.

\bibitem{hair-trigger}
Y.~F. Bakman, D.~N. Yaldiz, S.~Avestimehr, and S.~P. Karimireddy,
  ``Hair-trigger alignment: Black-box evaluation cannot guarantee post-update
  alignment,'' in \emph{Proceedings of the 6th Workshop on Trustworthy NLP
  (TrustNLP 2026)}, 2026, pp. 180--203.

\bibitem{yang-yeung}
S.~Yang and Y.-H. Yeung, ``Unlearning as distribution restoration: A controlled
  counterfactual study, a validated selective screen, and the limits of
  oracle-free certification,'' \emph{arXiv preprint arXiv:2607.19442}, 2026.

\bibitem{testing-between}
M.~Ghalan, C.~Rodgers, and Z.~D. Asher, ``Testing between the test cases:
  Proving end-to-end steering in conditions you never drove,'' \emph{arXiv
  preprint arXiv:2609.10951}, 2026.

\bibitem{gif}
A.~Storek, N.~Holzer, Z.~Zhang, and S.~Jana, ``{GIF}: Locally sound geometric
  information flow control for {LLMs},'' \emph{arXiv preprint
  arXiv:2606.23277}, 2026.

\bibitem{repe}
A.~Zou, L.~Phan, S.~Chen, J.~Campbell, P.~Guo, R.~Ren, A.~Pan, X.~Yin,
  M.~Mazeika, A.-K. Dombrowski \emph{et~al.}, ``Representation engineering: A
  top-down approach to {AI} transparency,'' \emph{arXiv preprint
  arXiv:2310.01405}, 2023.

\bibitem{rogue-scalpel}
A.~Korznikov, A.~Galichin, A.~Dontsov, O.~Y. Rogov, I.~Oseledets, and
  E.~Tutubalina, ``The rogue scalpel: Activation steering compromises {LLM}
  safety,'' \emph{arXiv preprint arXiv:2509.22067}, 2025.

\bibitem{abliteration}
A.~Fafu{\l}a, ``Abliteration is not a scalpel: Off-target effects of refusal
  removal on decision disposition across model families,'' \emph{arXiv preprint
  arXiv:2607.17427}, 2026.

\bibitem{steering-pitfalls}
Y.~Li, A.~Fastowski, E.~Zaradoukas, B.~Prenkaj, and G.~Kasneci, ``Analysing the
  safety pitfalls of steering vectors,'' in \emph{Findings of the Association
  for Computational Linguistics: ACL 2026}, 2026, pp. 11\,182--11\,204.

\bibitem{nullspace-steering}
R.~Yin, T.~Han, N.~Xu, C.~Li, P.~He, C.~Zhou, J.~Wang, Z.~Fu, T.~Du, J.~Li
  \emph{et~al.}, ``Compiling activation steering into weights via null-space
  constraints for stealthy backdoors,'' in \emph{Proceedings of the 64th Annual
  Meeting of the Association for Computational Linguistics (Volume 1: Long
  Papers)}, 2026, pp. 26\,228--26\,245.

\bibitem{z3}
L.~de~Moura and N.~Bj{\o}rner, ``Z3: An efficient smt solver,'' in \emph{Tools
  and Algorithms for the Construction and Analysis of Systems}, C.~R.
  Ramakrishnan and J.~Rehof, Eds.\hskip 1em plus 0.5em minus 0.4em\relax
  Berlin, Heidelberg: Springer Berlin Heidelberg, 2008, pp. 337--340.

\bibitem{autolirpa}
K.~Xu, Z.~Shi, H.~Zhang, Y.~Wang, K.-W. Chang, M.~Huang, B.~Kailkhura, X.~Lin,
  and C.-J. Hsieh, ``Automatic perturbation analysis for scalable certified
  robustness and beyond ({auto\_LiRPA}),'' in \emph{Advances in Neural
  Information Processing Systems}, vol.~33, 2020.

\bibitem{crown}
H.~Zhang, T.-W. Weng, P.-Y. Chen, C.-J. Hsieh, and L.~Daniel, ``Efficient
  neural network robustness certification with general activation functions
  ({CROWN}),'' in \emph{Advances in Neural Information Processing Systems},
  vol.~31, 2018.

\end{thebibliography}

\appendices
\section{Full Proofs of P1--P4}
\label{app:proofs}

Throughout, $g$ is a network of Definition~\ref{def:model}, so $g_h$ is
continuous and piecewise-linear on $\D=[0,1]^d$; ``claim'', ``certificate'',
``edit'', ``inflation'' and ``certified radius'' are as in
Definitions~\ref{def:claim}--\ref{def:radius}.

\subsection{Proposition~\ref{prop:p1} (region monotonicity and union closure)}

\begin{proof}
Both parts are immediate from the meaning of the universal quantifier. (a) A
statement true of every point of $R$ is true of every point of any $R'\subseteq
R$. (b) A statement true of every point of each $R_i$ is true of every point of
$\bigcup_i R_i$. Consequently the union of \emph{all} certified regions is itself
certified and contains every certified region; it equals the truth set
$\mathcal{T}=\{x\in\D:g_h(x)\rhd 0\}$ because each singleton $\{x\}\subseteq \mathcal{T}$ is
certified, and any certified region is a subset of $\mathcal{T}$ by definition.
\end{proof}

\subsection{Proposition~\ref{prop:p2} (the certified radius)}

\begin{proof}
Write $\mathrm{lo}_i(\eps)=\max(0,l_i-\eps)$, $\mathrm{hi}_i(\eps)=\min(1,u_i+\eps)$.

\emph{(a) Interval structure.} If $\eps'\le\eps$ then $\mathrm{lo}_i(\eps)\le
\mathrm{lo}_i(\eps')$ and $\mathrm{hi}_i(\eps')\le\mathrm{hi}_i(\eps)$
coordinatewise, so $R_{\eps'}\subseteq R_\eps$; the claim on $R_\eps$ then gives
the claim on $R_{\eps'}$ by Proposition~\ref{prop:p1}(a). Hence $S$ is downward
closed in $[0,\eps_{\max}]$ and contains $0$, i.e.\ an interval.

\emph{(b) Attainment.} Non-strict sign $\le$: fix $x\in R_{\epsstar}$; we show
$g_h(x)\le0$. Each endpoint is $1$-Lipschitz in $\eps$ and
$\mathrm{lo}_i(\eps)\le l_i\le u_i\le\mathrm{hi}_i(\eps)$, so every $R_\eps$ is a
nonempty box. For $\eps<\epsstar$ let $x(\eps)$ be the coordinatewise clip of $x$
into $R_\eps$; then $x(\eps)\in R_\eps$ and $|x(\eps)_i-x_i|\le\epsstar-\eps$, so
$x(\eps)\to x$ as $\eps\uparrow\epsstar$. Choose $\eps_n\uparrow\epsstar$ with
$\eps_n\in S$ (possible since $S\supseteq[0,\epsstar)$ by (a)); then
$g_h(x(\eps_n))\le0$, and by continuity $g_h(x)=\lim_n g_h(x(\eps_n))\le0$. Hence
$\epsstar\in S$. Strict sign $>$: $[0,\epsstar)\subseteq S$ is (a); that
$\epsstar$ may fail is shown by $d=1$, $g(x)=\relu(0.5-x)$, claim $g>0$ on
$R=[0,0.2]$: then $R_\eps=[0,0.2+\eps]$ and the claim holds iff $0.2+\eps<0.5$,
so $S=[0,0.3)$ and at $\epsstar=0.3$ the point $x=0.5$ gives $g(x)=0\not>0$.

\emph{(c) Decidability.} The negated claim is
$\exists x:\bigwedge_i(\mathrm{lo}_i\le x_i\le\mathrm{hi}_i)\wedge\lnot(g_h(x)\rhd
0)$, where $g_h(x)$ is built from rational constants, addition,
multiplication-by-constant, and $\mathrm{ite}(z>0,z,0)$ terms. Eliminating each
$\mathrm{ite}$ into a disjunction of two linear cases yields a quantifier-free
formula of linear real arithmetic; satisfiability of QF\_LRA is decidable (e.g.\
by simplex-based DPLL(T)), and with rational inputs the answer is exact. A
solver timeout returns ``unknown'', which the implementation reports
rather than treating as either verdict.

\emph{(d) Cost bound.} One query at $\eps=0$ (vacuity) and one at $\eps_{\max}$
(saturation); thereafter bisection maintains a bracket $[\ell,\mathrm{hi}]$ with
$\ell\in S$, $\mathrm{hi}\notin S$---valid because $S$ is an interval by
(a)---halving $\mathrm{hi}-\ell$ once per query until $\mathrm{hi}-\ell\le\tau$,
i.e.\ $\lceil\log_2(\eps_{\max}/\tau)\rceil$ queries. Since
$\ell\le\epsstar\le\mathrm{hi}$, $|\ell-\epsstar|\le\tau$. The bracket half of
the guarantee assumes each probe is answered; treating ``unknown'' as a failed
probe preserves soundness of $\ell\in S$ but voids $|\ell-\epsstar|\le\tau$.
\end{proof}

\subsection{Proposition~\ref{prop:p3} and its corollaries}

\begin{proof}[Proof of Proposition~\ref{prop:p3}]
Constructive. \emph{The gadget.} The finitely many test points have finitely
many first coordinates $\{t_0:t\in T\}$; together with $l_0,u_0$ these leave an
open interval $(c-w,c+w)\subset(l_0,u_0)$ disjoint from all of them (finitely
many points cannot cover an interval: pick $c$ strictly between two consecutive
such values and $w$ below half the gap). Shrink $w$ further if needed so
$w<\delta(u_0-l_0)$. Take three hidden neurons with pre-activations
$x_0-(c-w),\,x_0-c,\,x_0-(c+w)$ and feed head $A$
\[
  \mathrm{tent}_\beta(x_0)=\beta\bigl[\relu(x_0-(c{-}w))-2\relu(x_0-c)+\relu(x_0-(c{+}w))\bigr].
\]
A case check on the four intervals cut by $c-w,c,c+w$ shows $\mathrm{tent}_\beta$
is the ``hat'': zero for $x_0\le c-w$, rising linearly to peak $\beta w$ at
$x_0=c$, falling to zero at $x_0=c+w$, and \emph{identically zero} thereafter
(the slopes $\beta,-2\beta,\beta$ cancel). In particular
$\mathrm{tent}_\beta\ge0$ and $\mathrm{tent}_\beta(t_0)=0$ at every test point.

\emph{The model.} Let $f$ consist of a main pathway for skill $A$ with decision
margin exceeding $\beta w/2$ on $A$'s specification regions; an untouched
pathway for skill $B$ reading only $x_1$ (zero weight from the gadget into head
$B$); the gadget neurons into head $A$; and head-$A$ bias lowered by $\beta w/2$
(the ``hurdle''). Let $C$ be the main-pathway $A$-neurons and $g=\ablate(f,C)$.

(a) Pre-edit, head $A=\mathrm{main}(x_0)+\mathrm{tent}_\beta(x_0)-\beta w/2$. On
$A$'s HIGH region the main margin exceeds $\beta w/2$ and $\mathrm{tent}_\beta\ge
0$, so head $A>0$; on $A$'s LOW region the bump is absent (it lies inside HIGH,
so $\mathrm{tent}_\beta=0$ there), the main pathway is nonpositive, and the
hurdle only lowers it, so head $A\le0$. Skill $B$'s claims hold as its pathway is
disjoint from the gadget and $C$.

(b) Post-edit, $g_A(x)=\mathrm{tent}_\beta(x_0)-\beta w/2$. At each $t\in T$,
$\mathrm{tent}_\beta(t_0)=0$, so $g_A(t)=-\beta w/2<0$; and $B$'s pathway,
untouched and receiving nothing from the gadget, preserves both claims on their
full regions.

(c) $g_A(x)>0$ iff $\mathrm{tent}_\beta(x_0)>\beta w/2$ iff $|x_0-c|<w/2$; so
$S_g=\{x\in R:x_0\in(c-w/2,c+w/2)\}$ is open, nonempty, of relative volume
$w/(u_0-l_0)<\delta$.
\end{proof}

\begin{corollary}[random testing]
\label{cor:c31}
A test drawing $m$ i.i.d.\ uniform points from $R$ wrongly approves $g$ with
probability at least $(1-\delta)^m\ge1-m\delta$.
\end{corollary}

\begin{proof}
By Proposition~\ref{prop:p3}(c), each draw lands in $S_g$ with probability below
$\delta$; independence gives miss probability $>(1-\delta)^m\ge1-m\delta$ (Bernoulli). For
the constructed gadget the pocket volume is exact; for the naturally-trained
illusions of \S\ref{sec:findings-illusion} it is only sampled (zero hits in
$2\times10^6$ darts, a $95\%$ Clopper--Pearson ceiling of $1.5\times10^{-6}$).
\end{proof}

\begin{proof}[Proof of Corollary~\ref{cor:c32}]
Run the deterministic protocol $P$ against a genuinely edited $g_0$
($g_{0,A}\le0$ on all of $R$); it queries a finite set $T'$ and approves. Apply
the gadget to $T'$: choose $(c-w,c+w)$ avoiding $T'$'s first coordinates and add
the three tent neurons to head $A$ with pattern $(\beta,-2\beta,\beta)$ and no
bias change, giving $g=g_0+\mathrm{tent}_\beta$ on head $A$. Since
$\mathrm{tent}_\beta$ vanishes at every point of $T'$, $g$ agrees with $g_0$ at
every query, so $P$'s transcript---and hence its verdict, $P$ being
deterministic---is identical: $P$ approves $g$. But $g_{0,A}$ is continuous on
the compact $R$, hence bounded below by some $-G$; choosing $\beta$ with $\beta
w/2>G$ makes $g_A>0$ on $|x_0-c|<w/2$, a positive-measure set. ($g$ is a
legitimate edited model: add the same gadget to $g_0$'s parent and ablate the
same $C$.)
\end{proof}

\noindent\emph{Scope.} Corollary~\ref{cor:c31} covers i.i.d.\ uniform sampling;
Corollary~\ref{cor:c32} covers deterministic black-box protocols (adaptive
included). A \emph{randomized} adaptive tester (transcript replay yields only a
fooling probability) and a \emph{white-box} analysis such as the SMT solver used
here are deliberately outside the statements---the latter escape being the
paper's point.

\begin{table*}[t]
	\centering
	\caption{The command behind each reported result.}
	\label{tab:repro}
	\footnotesize
	\begin{tabular}{@{}p{0.64\textwidth}p{0.32\textwidth}@{}}
		\toprule
		\textbf{Result} & \textbf{Script} \\
		\midrule
		Softmax\,+\,LayerNorm transformer, $448$ noise dimensions: certified $\le$ attack for removal and both preservation claims (Fig.~\ref{fig:bp}) & \texttt{run\_boundprop\_transformer.py}$^{\dagger}$ \\
		Threshold-gate transformer, exact over all sequences: removal $0.016$, preservation $0.013$, collateral $0.034\to0.013$ (Table~\ref{tab:results}) & \texttt{run\_transformer.py} \\
		Distributed circuit: neither the attention heads nor the MLP alone removes skill~$A$ (Fig.~\ref{fig:circuit}) & \texttt{results/rung2\_circuit\_search.py}$^{b}$ \\
		Modular-adder transformer: removal as exact independence of the summands, radius ${\ge}0.05$ (\S\ref{sec:findings-positive}) & \texttt{run\_rung3.py} \\
		Modular-adder transformer: correctness on all $625$ clean sequences, preservation radius ${\ge}0.002$ (\S\ref{sec:findings-positive}) & \texttt{run\_rung3\_preservation.py} \\
		Constructed illusion: passes a $40{,}401$-point grid with preservation proved, refuted by the solver (Fig.~\ref{fig:illusion}a,b) & \texttt{run\_illusion.py} \\
		Three naturally trained illusions in $160$ four- and five-input models (\S\ref{sec:findings-illusion}) & \texttt{run\_illusion\_nd.py} \\
		Illusions per input dimension, $80$ models each; the $d=4,5$ rows reproduce the pooled run exactly (Fig.~\ref{fig:illusion}c) & \texttt{run\_illusion\_dims.py} \\
		Certified influence exactly $0$ on the transformer (\S\ref{sec:findings-influence}) & \texttt{run\_transformer\_independence.py} \\
		Certified influence on the separated and entangled models; the illusion edit refuted with an input pair (\S\ref{sec:findings-influence}) & \texttt{run\_independence.py} \\
		Certified boundary within $0.0012$ of the model's decision rule (\S\ref{sec:findings-radius}) & \texttt{run\_margin\_curve.py} \\
		Certified edge against the model's own decision boundary (Fig.~\ref{fig:radius}) & \texttt{run\_robustness.py} \\
		Radii per edit type on the two toy subjects (Table~\ref{tab:results}, \S\ref{sec:findings-edits}) & \texttt{run\_robustness.py} \\
		Deep MLP: four-neuron layer-1 circuit and its collateral (Table~\ref{tab:results}) & \texttt{run\_deep\_multi.py} \\
		Steering dose sweep over doses $0.5$--$64$ (\S\ref{sec:findings-edits}) & \texttt{run\_collateral\_sweep.py}$^{a}$ \\
		On the transformer, no steering dose both removes $A$ and preserves $B$ (\S\ref{sec:findings-edits}) & \texttt{run\_transformer\_steering.py} \\
		SAE feature-clamp edit certified by the same prover (\S\ref{sec:findings-edits}) & \texttt{run\_sae\_clamp.py} \\
		Edit-type tables on the transformer and the modular adder (\S\ref{sec:findings-edits}) & \texttt{run\_transformer\_edit\_table.py} \newline \texttt{run\_rung3\_edit\_table.py} \\
		Exact-frontier size ladder and encoding timeouts (\S\ref{sec:findings-reach}) & \texttt{results/rung2\_size\_ladder.py}$^{b}$ \newline \texttt{results/rung2\_hull\_wall.py}$^{b}$ \\
		M0: bound propagation equals the exact radius where both apply (\S\ref{sec:method-bp}) & \texttt{run\_boundprop\_validate.py}$^{\dagger}$ \\
		Smallest token swap is ${\approx}324\times$ the certified ball (\S\ref{sec:findings-reach}) & \texttt{run\_boundprop\_quantifier.py}$^{\dagger}$ \\
		Float-gap ceiling and slack re-proofs (App.~\ref{app:floatgap}) & \texttt{run\_float\_gap.py} \\
		Machine-checked witnesses for P1--P4 (App.~\ref{app:witnesses}) & \texttt{run\_theory\_checks.py} \\
		\bottomrule
	\end{tabular}

	\parbox{\textwidth}{\footnotesize\vspace{2pt}
		Each script \texttt{run\_}$\langle$\textit{name}$\rangle$\texttt{.py} writes \texttt{results/}$\langle$\textit{name}$\rangle$\texttt{\_report.\{md,json\}}, except:
		$^{a}$ writes \texttt{results/collateral\_sweep.json};
		$^{b}$ writes the matching \texttt{results/rung2\_*.log}.
		$^{\dagger}$ requires the separate bound-propagation environment; every other script needs only \texttt{numpy} and an SMT solver.}
\end{table*}

\subsection{Proposition~\ref{prop:p4} (edit-class separation)}

\begin{proof}
\emph{(a)} Neurons outside $C$ have identical pre-activations and outputs under
$\mathrm{steer}_s$ and $\mathrm{abl}$. A neuron $j\in C$ outputs
$\relu(\mathrm{pre}_j(x)-s)$ under $\mathrm{steer}_s$ and $0$ under
$\mathrm{abl}$, where $\mathrm{pre}_j$ is the \emph{unedited} pre-activation.
Heads are affine in the neuron outputs with coefficients $W_2[h,\cdot]$, so the
head difference is exactly $\sum_{j\in C}W_2[h,j]\,\relu(\mathrm{pre}_j(x)-s)$.

\emph{(b)} $\mathrm{pre}_j$ is affine, so its maximum over the box $\D$ is at a
corner, giving $M_j(\D)$. If $s\ge M_j(\D)$ for all $j\in C$ then every leftover
term is $0$ on $\D$, so by (a) the two edited models are equal as functions;
equal functions satisfy the same claims over every region, hence share every
certificate and certified radius.

\emph{(c)} Let $F_R(s)=\max_{x\in R}\mathrm{steer}_{s,A}(x)$, so that
$S(R)=\{s\ge0: F_R(s)\le0\}$. By (a), $\mathrm{steer}_{s,A}(x)$ is jointly
continuous in $(x,s)$, and $R$ is compact, so $F_R$ is continuous and $S(R)$
is closed. For $s\ge\max_{j\in C}M_j(\D)$, (b) gives
$\mathrm{steer}_s\equiv\mathrm{abl}$ on $\D$, and $\mathrm{abl}$ certifies
removal over $R$ by hypothesis, so $S(R)$ contains this half-line. Hence $S(R)$
is nonempty, closed and bounded below by $0$, and its minimum $s^*(R)$ is
attained. If $R\subseteq R'$ then $F_R\le F_{R'}$ pointwise, so
$S(R')\subseteq S(R)$ and $s^*(R)\le s^*(R')$; $R_\eps$ is monotone in $\eps$
(Proposition~\ref{prop:p2}a), so $s^*(R_\eps)$ is non-decreasing in $\eps$.
If $W_2[A,j]\ge0$ for every $j\in C$, then each leftover term
$W_2[A,j]\relu(\mathrm{pre}_j(x)-s)$ is non-increasing in $s$, so
$\mathrm{steer}_{s,A}(x)$ and therefore $F_R$ are non-increasing in $s$, and
$S(R)$ is up-closed, i.e.\ $S(R)=[s^*(R),\infty)$.

The sign condition cannot be dropped. Take $\D=[0,1]^2$, $R=R_A$, three neurons
with zero input weights and biases $1,2,3$ (so $\mathrm{pre}_j\equiv j$),
$W_2[A,\cdot]=(4,-4,2)$ and head-$A$ bias $-1$, and $C=\{1,2,3\}$. The unedited
head is $g_A\equiv1>0$, so the skill is present, and $\mathrm{abl}_A\equiv-1$
certifies removal. By (a),
$\mathrm{steer}_{s,A}\equiv-1+4\relu(1-s)-4\relu(2-s)+2\relu(3-s)$, which is
$\le0$ exactly for $s\in[\tfrac12,\tfrac32]\cup[\tfrac52,\infty)$. The
certifying doses are therefore not an interval: dose $1$ certifies removal,
dose $2$ does not. Once dose $s\ge1$ silences neuron~$1$, further dose only
shrinks neuron~$2$'s negative contribution, so the skill returns on
$(\tfrac32,\tfrac52)$ before the suppression of neuron~$3$ removes it again.
The SMT solver reproduces this set exactly on a grid of doses (App.~\ref{app:witnesses}).

\emph{(d)} Applying the same affine-head argument to $\mathrm{steer}_v$ and the
unedited $g$ gives
\[
	\mathrm{steer}_{v,h}(x)-g_h(x)=\sum_j W_2[h,j]\bigl[\relu(\mathrm{pre}_j(x)+v_j)-\relu(\mathrm{pre}_j(x))\bigr].
\]
Fix $x$ and write $v=s\hat v$. Terms with $\hat v_j=0$ vanish. For $j$ with
$\hat v_j\neq0$, the definition of $\bar s(x)$ as an infimum means that the status of
$\mathrm{pre}_j(x)+s\hat v_j$ (active iff $>0$) is the same as that of
$\mathrm{pre}_j(x)$ for every $s\in[0,\bar s(x))$. On that range a neuron in
$\mathcal A(x)$ contributes $\relu(\mathrm{pre}_j+s\hat v_j)-\relu(\mathrm{pre}_j)=s\hat v_j$,
and an inactive one contributes $0$. So the head-$B$ difference equals
$s\,\kappa(x)$. Both sides are continuous in $s$, so the equality also holds at
$s=\bar s(x)$ when it is finite. For larger $s$, each term
$\relu(\mathrm{pre}_j(x)+s\hat v_j)$ is piecewise linear in $s$ with one
breakpoint, so the difference is continuous and piecewise linear, with at most
$|\{j:\hat v_j\neq0\}|$ breakpoints. If $W_2[B,j]=0$ whenever $\hat v_j\neq0$,
every term of the identity vanishes for all $x$ and $s$. Finally, if
$\kappa(x)<0$ and $g_B(x)/|\kappa(x)|\le s\le\bar s(x)$, then
$\mathrm{steer}_{v,B}(x)=g_B(x)+s\,\kappa(x)\le0$, which violates $(B,R,>)$ at
$x\in R$. Because $\bar s(x)$ varies with $x$, no single dose range makes the
shift linear on all of $R$ in general.
\end{proof}

\section{Exact Encoding Details and the Float Gap}
\label{app:floatgap}
The encoding converts every weight to an exact rational and every ReLU to an
$\mathrm{ite}$ term; a claim is proved by \texttt{unsat} of its negation. The
forward-error ceiling between the ideal rational network and the executed
float64 program is $\approx10^{-13}$ for our models. The toy model's removal,
preservation and control certificates re-prove with slack $\gamma=10^{-9}$
(\texttt{prove\_forall(\dots,slack=$\gamma$)}), and every threshold-gate
transformer certificate carries slack $10^{-6}$. A
constructed zero-margin claim proved exactly yet violated by the float program by
$\sim\!10^{-16}$ on about a quarter of a region is the boundary of the
transfer argument; region endpoints are float constants too ($\sim\!10^{-17}$
from the stated rationals), immaterial at our margins.

\section{Bound-Propagation Setup}
\label{app:bp}
CROWN bounds are computed in chunks with gradients disabled during
certification for memory safety; radii are reported over balls of radius
$\ge10^{-4}$ (CROWN is degenerate at $\eps=0$). Every certified radius is
bracketed above by a projected-gradient attack, and the bound-propagation
dependencies are isolated from the pure \texttt{numpy}$+$SMT exact pipeline.

\section{Reproducibility}
\label{app:repro}
Each reported result is regenerated by one command from committed code and fixed seeds; Table~\ref{tab:repro} maps every result to its script.

\section{Machine-Checked Witnesses}
\label{app:witnesses}
P2--P4 each have a runnable check (P1 follows from the meaning of the universal quantifier): P2's query-count bound with re-proof at
the returned radius and re-refutation just past it; P3's per-grid gadget with the
survivor located in the predicted sliver $c\pm w/2$; the adaptive-tester
transcript replay for Corollary~\ref{cor:c32}; P4(a)'s leftover identity at
$2\times10^5$ random inputs and several doses, P4(b)'s corner-formula threshold,
P4(c)'s dose set (on a trained model whose circuit feeds head $A$ positively, a
dose sweep never certifies a dose and then refutes a larger one; on the
mixed-sign counterexample above, the solver's verdicts at doses
$0,\tfrac14,\dots,4$ match $[\tfrac12,\tfrac32]\cup[\tfrac52,\infty)$),
and P4(d)'s collateral formula on every sampled input whose activation pattern
the dose leaves unchanged (doses $0.5$--$4$; at dose $8$ no sampled input
qualifies).
All checks pass.

\end{document}